\documentclass{article}[11pt]
\usepackage{fullpage}

\usepackage[utf8]{inputenc}
\usepackage{hyperref}
\usepackage{amsfonts}
\usepackage{amsmath}
\usepackage{amssymb}
\usepackage{amsthm}
\usepackage{qtree}
\usepackage{graphicx}
\graphicspath{ {./figures/} }
\usepackage{xcolor}
\usepackage[english]{babel}
\DeclareMathOperator*{\argmax}{arg\,max}

\usepackage[ruled,vlined]{algorithm2e}
\usepackage{bbm}

\theoremstyle{definition}
\newtheorem{definition}{Definition}[section]  
\theoremstyle{remark}
\newtheorem*{remark}{Remark}
\theoremstyle{plain}
\newtheorem*{theorem*}{Theorem}
\newtheorem{theorem}{Theorem}

\definecolor{DarkGreen}{rgb}{0.1,0.5,0.1}
\definecolor{DarkRed}{rgb}{0.5,0.1,0.1}
\definecolor{DarkBlue}{rgb}{0.1,0.1,0.5}
\hypersetup{
    unicode=false,          
    pdftoolbar=true,        
    pdfmenubar=true,        
    pdffitwindow=false,     
    pdfnewwindow=true,      
    colorlinks=true,       
    linkcolor=DarkGreen,    
    citecolor=DarkGreen,    
    filecolor=DarkGreen,    
    urlcolor=DarkBlue,      
}

\usepackage{natbib}
\setcitestyle{authoryear,round,citesep={;},aysep={,},yysep={;}}

\title{The Limitations of Limited Context for Constituency Parsing} 

\author{Yuchen Li \\
  Carnegie Mellon University \\
  \texttt{yuchenl4@andrew.cmu.edu} \\
  \and
  Andrej Risteski \\
  Carnegie Mellon University \\
  \texttt{aristesk@andrew.cmu.edu} \\}

\date{}

\begin{document}

\maketitle

% Abstract
\begin{abstract}

Incorporating syntax into neural approaches in NLP has a multitude of practical and scientific benefits. For instance, a language model that is syntax-aware is likely to be able to produce better samples; even a discriminative model like BERT with a syntax module could be used for core NLP tasks like unsupervised syntactic parsing.  Rapid progress in recent years was arguably spurred on by the empirical success of the Parsing-Reading-Predict architecture of \citep{shen2018neural}, later simplified by the Order Neuron LSTM of \citep{shen2018ordered}. Most notably, this is the first time neural approaches were able to successfully perform unsupervised syntactic parsing (evaluated by various metrics like F-1 score). 

However, even heuristic (much less fully mathematical) understanding  of why and when these architectures work is lagging severely behind. In this work, we answer representational questions raised by the architectures in \citep{shen2018neural, shen2018ordered}, as well as some transition-based syntax-aware language models \citep{dyer2016recurrent}: \emph{what kind of syntactic structure can current neural approaches to syntax represent}? Concretely, we ground this question in the sandbox of probabilistic context-free-grammars (PCFGs), and identify a key aspect of the representational power of these approaches: the \emph{amount} and \emph{directionality} of context that the predictor has access to when forced to make parsing decision. We show that with \emph{limited context} (either bounded, or unidirectional), there are PCFGs, for which these approaches cannot represent the max-likelihood parse; conversely, if the context is \emph{unlimited}, they can represent the max-likelihood parse of any PCFG.

\end{abstract}

% Intro
\section{Introduction} \label{sec:intro}

Neural approaches have been steadily making their way to NLP in recent years. By and large however, the neural techniques that have been scaled-up the most and receive widespread usage do not explicitly try to encode discrete structure that is natural to language, e.g. syntax. The reason for this is perhaps not surprising: neural models have largely achieved substantial improvements in \emph{unsupervised} settings, BERT \citep{devlin-etal-2019-bert} being the de-facto method for unsupervised pre-training in most NLP settings. On the other hand unsupervised \emph{syntactic} tasks, e.g. unsupervised syntactic parsing, have long been known to be very difficult tasks \citep{htut2018grammar}.
 However, since incorporating  syntax has been shown to improve language modeling \citep{kim2019unsupervised} as well as natural language inference \citep{chen-etal-2017-enhanced, pang2019improving, he-etal-2020-enhancing}, syntactic parsing remains important even in the current era when large pre-trained models, like BERT \citep{devlin-etal-2019-bert}, are available. 

Arguably, the breakthrough works in unsupervised constituency parsing in a neural manner were \citep{shen2018neural, shen2018ordered}, achieving F1 scores 42.8 and 49.4 on the WSJ Penn Treebank dataset \citep{htut2018grammar, shen2018ordered}. Both of these architectures, however (especially \citealp{shen2018neural}) are quite intricate, and it's difficult to evaluate what their representational power is (i.e. what \emph{kinds} of structure can they recover). 
Moreover, as subsequent more thorough evaluations show \citep{kim2019unsupervised, kim2019compound}, these methods still have a rather large performance gap with the oracle binary tree (which is the best binary parse tree according to F1-score) --- raising the question of \emph{what is missing} in these methods. 

We theoretically answer both questions raised in the prior paragraph. 
We quantify the representational power of two major frameworks in neural approaches to syntax: \emph{learning a syntactic distance} \citep{shen2018neural, shen2018straight, shen2018ordered} and \emph{learning to parse through sequential transitions} \citep{dyer2016recurrent, chelba1997structured}. To formalize our results, we consider the well-established sandbox of \emph{probabilistic context-free grammars} (PCFGs). 
Namely, we ask:

\emph{When is a neural model based on a syntactic distance or transitions able to represent the max-likelihood parse of a sentence generated from a PCFG?} 

We focus on a crucial ``hyperparameter'' common to practical implementations of both families of methods that turns out to govern the representational power: the amount and type of context the model is allowed to use when making its predictions. 
Briefly, for every position $t$ in the sentence, syntactic distance models learn a distance $d_t$ to the previous token --- the tree is then inferred from this distance;  
transition-based models iteratively construct the parse tree by deciding, at each position $t$, what operations to perform on a partial parse up to token $t$. A salient feature of both is the \emph{context}, that is, \emph{which tokens is }$d_t$ \emph{a function of}  (correspondingly, which tokens can the choice of operations at token $t$ depend on)? 

We show that when the context is either \emph{bounded} (that is, $d_t$ only depends on a bounded window around the $t$-th token) or \emph{unidirectional} (that is, $d_t$ only considers the tokens to the left of the $t$-th token), there are PCFGs for which no distance metric (correspondingly, no algorithm to choose the sequence of transitions) works. On the other hand, if the context is \emph{unbounded in both directions} then both methods work: that is, for any parse, we can design a distance metric (correspondingly, a sequence of transitions) 
that recovers it.

This is of considerable importance: in practical implementations the context is either bounded (e.g. in \citealp{shen2018neural}, the distance metric is parametrized by a convolutional kernel with a constant width) or unidirectional (e.g. in \citealp{shen2018ordered}, the distance metric is computed by a LSTM, which performs a left-to-right computation). 

This formally confirms a conjecture of \citet{htut2018grammar}, who suggested that because these models commit to parsing decision in a left-to-right fashion and are trained as a part of a language model, it may be difficult for them to capture sufficiently complex syntactic dependencies. Our techniques are fairly generic and seem amenable to analyzing other approaches to syntax. 
Finally, while the existence of \emph{a particular} PCFG that is problematic for these methods doesn't necessarily imply that the difficulties will carry over to real-life data, the PCFGs that are used in our proofs closely track linguistic intuitions about difficult syntactic structures to infer: the parse depends on words that come much later in the sentence.

% Overview of Results
\section{Overview of Results} \label{sec:overview}

We consider several neural architectures that have shown success in various syntactic tasks, most notably unsupervised constituency parsing and syntax-aware language modeling. 
The general framework these architectures fall under is as follows: to parse a sentence $W = w_1 w_2 ... w_n$ with a trained neural model, the sentence $W$ is input into the model, which outputs $o_t$ at each step $t$, and finally all the outputs $\{o_t\}_{t=1}^n$ are utilized to produce the parse. 

Given unbounded time and space resources, by a seminal result of \citet{siegelmann1992rnn}, an RNN implementation of this framework is Turing complete. 
In practice it is common to restrict the form of the output $o_t$ in some way. In this paper, we consider the two most common approaches, in which $o_t$ is 
a \emph{real number representing a syntactic distance} (Section \ref{sec:overview:distance}) \citep{shen2018neural, shen2018straight, shen2018ordered} or a \emph{sequence of parsing operations} (Section \ref{sec:overview:seq_model}) \citep{chelba1997structured, chelba2000structured, dyer2016recurrent}. We proceed to describe our results for each architecture in turn.

\subsection{Syntactic distance} \label{sec:overview:distance}

\emph{Syntactic distance}-based neural parsers train a neural network to learn a distance for each pair of adjacent words, depending on the context surrounding the pair of words under consideration.
The distances are then used to induce a tree structure \citep{shen2018neural, shen2018straight}.

For a sentence $W = w_1 w_2 ... w_n$, the syntactic distance between $w_{t-1}$ and $w_t$ ($2 \le t \le n$) is defined as 
$d_t = d(w_{t-1}, w_t \, | \, c_t)$, where 
$c_t$ is the context that $d_t$ takes into consideration 
\footnote{Note that this is not a conditional distribution---we use this notation for convenience.}.
We will show that restricting the surrounding context either in directionality, or in size, results in a poor representational power, while full context confers essentially perfect representational power with respect to PCFGs.  

Concretely, if the context is full, we show: 
\begin{theorem*}[Informal, full context]
    \label{thm:distance_positive_informal}
    For sentence $W$ generated by any PCFG,
    if the computation of $d_t$ has as context the full sentence and the position index under consideration, i.e. $c_t = (W, t)$ and $d_t = d(w_{t-1}, w_t \, | \, c_t)$, 
    then $d_t$ can induce the maximum likelihood parse of $W$. 
\end{theorem*}

On the flipside, if the context is unidirectional (i.e. unbounded left-context from the start of the sentence, and even possibly with a bounded look-ahead), the representational power becomes severely impoverished:
\begin{theorem*}[Informal, limitation of left-to-right parsing via syntactic distance]
    There exists a PCFG $G$ such that 
    for any distance measure $d_t$ whose computation incorporates only bounded context in at least one direction (left or right), e.g.
    \vspace{-4mm}
    \begin{align*}
        c_t &= (w_0, w_1, ..., w_{t+L'}) \\
        d_t &= d(w_{t-1}, w_t \, | \, c_t) \\
    \end{align*}
    \vspace{-12mm}
    
    the probability that $d_t$ induces the max likelihood parse is arbitrarily low. 
\end{theorem*}

In practice, for computational efficiency, parametrizations 
of syntactic distances fall into the above assumptions of restricted context \citep{shen2018neural}. 
This puts the ability of these models to learn a complex PCFG syntax into considerable doubt. 
For formal definitions, see Section \ref{sec:prelim:distance}.
For formal theorem statements and proofs, see Section \ref{sec:distance}.

Subsequently we consider ON-LSTM, an architecture proposed by \citet{shen2018ordered} improving their previous work \citep{shen2018neural}, 
which also is based on learning a syntactic distance, but in \citep{shen2018ordered} the distances are reduced from the values of a carefully structured master forget gate 
(see Section \ref{sec:on-lstm}).
While we show ON-LSTM can \emph{in principle} losslessly represent any parse tree (Theorem \ref{thm:on-lstm_lossless}), 
calculating the gate values in a left to right fashion 
(as is done in practice)
is subject to the same limitations as the syntactic distance approach: 

\begin{theorem*}[Informal, limitation of syntactic distance estimation based on ON-LSTM]
    There exists a PCFG $G$ for which the probability that 
    the syntactic distance converted from an ON-LSTM induces the max likelihood parse is arbitrarily low.
\end{theorem*}

For a formal statement, see Section \ref{sec:on-lstm} and in particular Theorem \ref{thm:on-lstm_to_distance}.

\subsection{Transition-based parsing} \label{sec:overview:seq_model}

In principle, the output $o_t$ at each position $t$ of a left-to-right neural models for syntactic parsing need not be restricted to a real-numbered distance or a carefully structured vector.
It can also be a combinatorial structure --- e.g. a sequence of 
\emph{transitions} \citep{chelba1997structured, chelba2000structured, dyer2016recurrent}.
We adopt a simplification of the neural parameterization in \citep{dyer2016recurrent} (see Definition \ref{def:adjoin}).

With full context, \citet{dyer2016recurrent} describes an algorithm to find a sequence of transitions to represent any parse tree, via a ``depth-first, left-to-right traversal" of the tree.
On the other hand, without full context, we prove that transition-based parsing suffers from the same limitations:

\begin{theorem*}[Informal, limitation of transition-based parsing without full context]
    There exists a PCFG $G$, 
    such that for any learned transition-based parser with bounded context in at least one direction (left or right),
    the probability that it returns the max likelihood parse is arbitrarily low.
\end{theorem*}

For a formal statement, see Section \ref{sec:seq_model}, and in particular Theorem \ref{thm:seq_left_to_right_backtrack}.

\begin{remark}
    There is no immediate connection between the syntactic distance-based approaches (including ON-LSTM) and the transition-based parsing framework, so 
    the limitations of transition-based parsing 
    does not directly imply the stated negative results for syntactic distance or ON-LSTM,
    and vice versa. 
\end{remark}

\subsection{The counterexample family}
\label{sec:overview:counterexample}

\begin{figure*}[t]
    \centering
    \includegraphics[width=10cm]{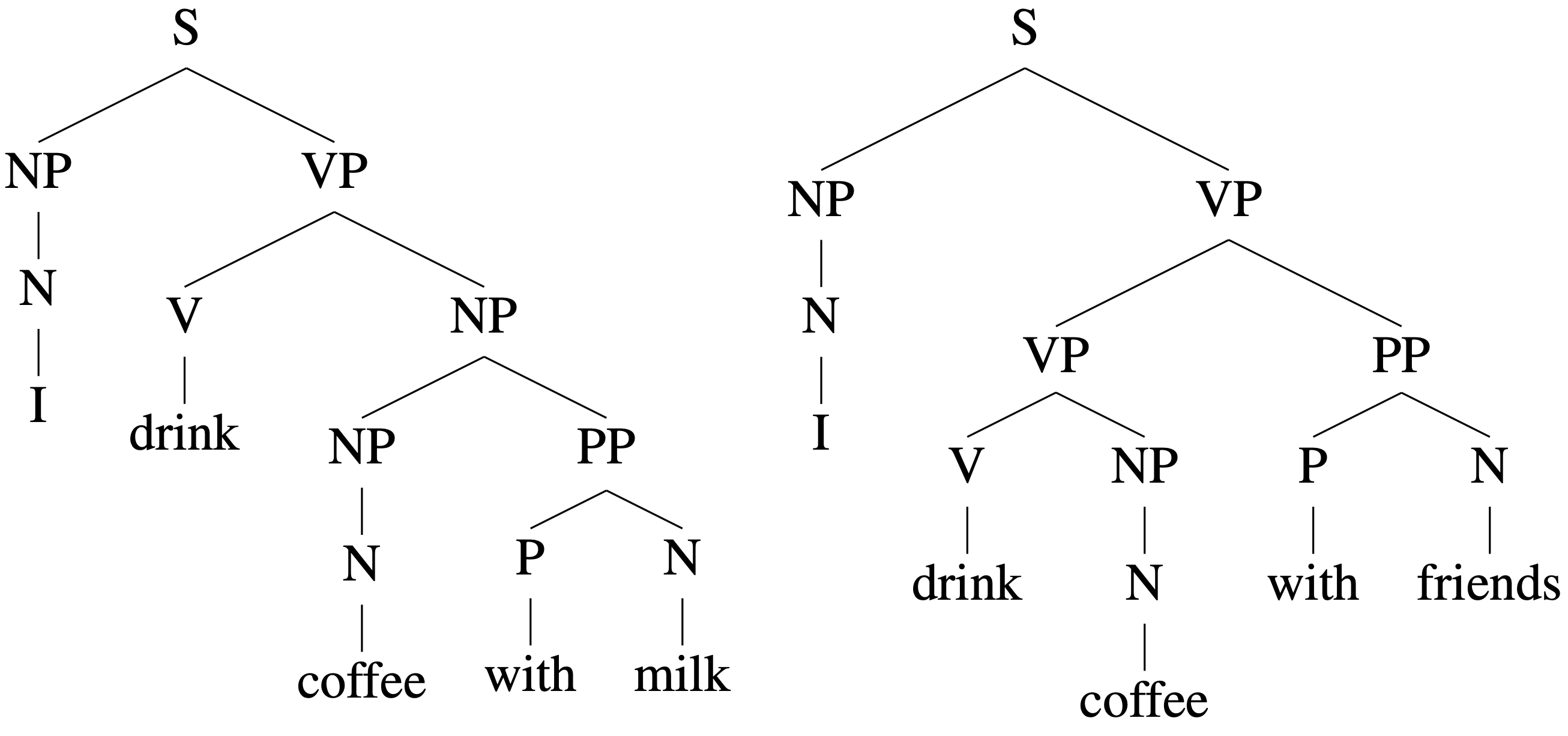}
    \caption{The parse trees of the two sentences: ``\textit{I drink coffee with milk.}'' and ``\textit{I drink coffee with friends.}''. 
    Their only difference occurs at their very last words, but their parses differ at some earlier words in each sentence
    }
    \label{fig:counterexample_motivation}
\end{figure*}

Most of our theorems proving limitations on bounded and unidirectional context are based on a PCFG family (Definition \ref{def:difficult_grammar}) which draws inspirations from natural language already suggested in \citep{htut2018grammar}: later words in a sentence can force different syntactic structures earlier in the sentence.
For example, consider the two sentences: ``\textit{I drink coffee with milk.}'' and ``\textit{I drink coffee with friends.}''
Their only difference occurs at their very last words, but their parses differ at some earlier words in each sentence, too, as shown in Figure \ref{fig:counterexample_motivation}.

To formalize this intuition, we define the following PCFG.

\begin{definition}[Right-influenced PCFG] \label{def:difficult_grammar}
    Let $m \ge 2, L' \ge 1$ be positive integers. 
    The grammar $G_{m,L'}$ has starting symbol $S$, other non-terminals 
    $$A_k, B_k, A_k^l, A_k^r, B_k' \text{ for all } k \in \{1, 2, ..., m\},$$ 
    and terminals 
    $$a_i \text{ for all } i \in \{1, 2, ..., m+1+L'\},$$ 
    $$c_j \text{ for all } j \in \{1, 2, ..., m\}.$$ 
    The rules of the grammar are 
    \begin{align*}
        S &\rightarrow A_k B_k, \forall k \in \{1,2, \dots, m\} \; \text{w. prob.} 1/m \\
        A_k &\rightarrow A_k^l A_k^r \quad \text{w. prob. } 1 \\
        A_k^l &\rightarrow^* a_1 a_2 ... a_k \quad \text{w. prob. } 1 \\
        A_k^r &\rightarrow^* a_{k+1} a_{k+2} ... a_{m+1} \quad \text{w. prob. } 1 \\
        B_k &\rightarrow^* B_k' c_k \quad \text{w. prob. } 1 \\
        B_k' &\rightarrow^* a_{m+2} a_{m+3} ... a_{m+1+L'} \quad \text{w. prob. } 1 \\
    \end{align*}
    in which $\rightarrow^*$ means that the left expands into the right through a sequence of rules that conform to the requirements of the Chomsky normal form (CNF, Definition \ref{def:cnf}).
    Hence the grammar $G_{m,L'}$ is in CNF.
   
   The language of this grammar is
    $$L(G_{m,L'}) {=} \{l_k {=} a_1 a_2 ... a_{m+1+L'} c_k : 1 \le k \le m \}.$$
\end{definition}

\begin{figure*}[t]
    \centering
    \includegraphics[width=10cm]{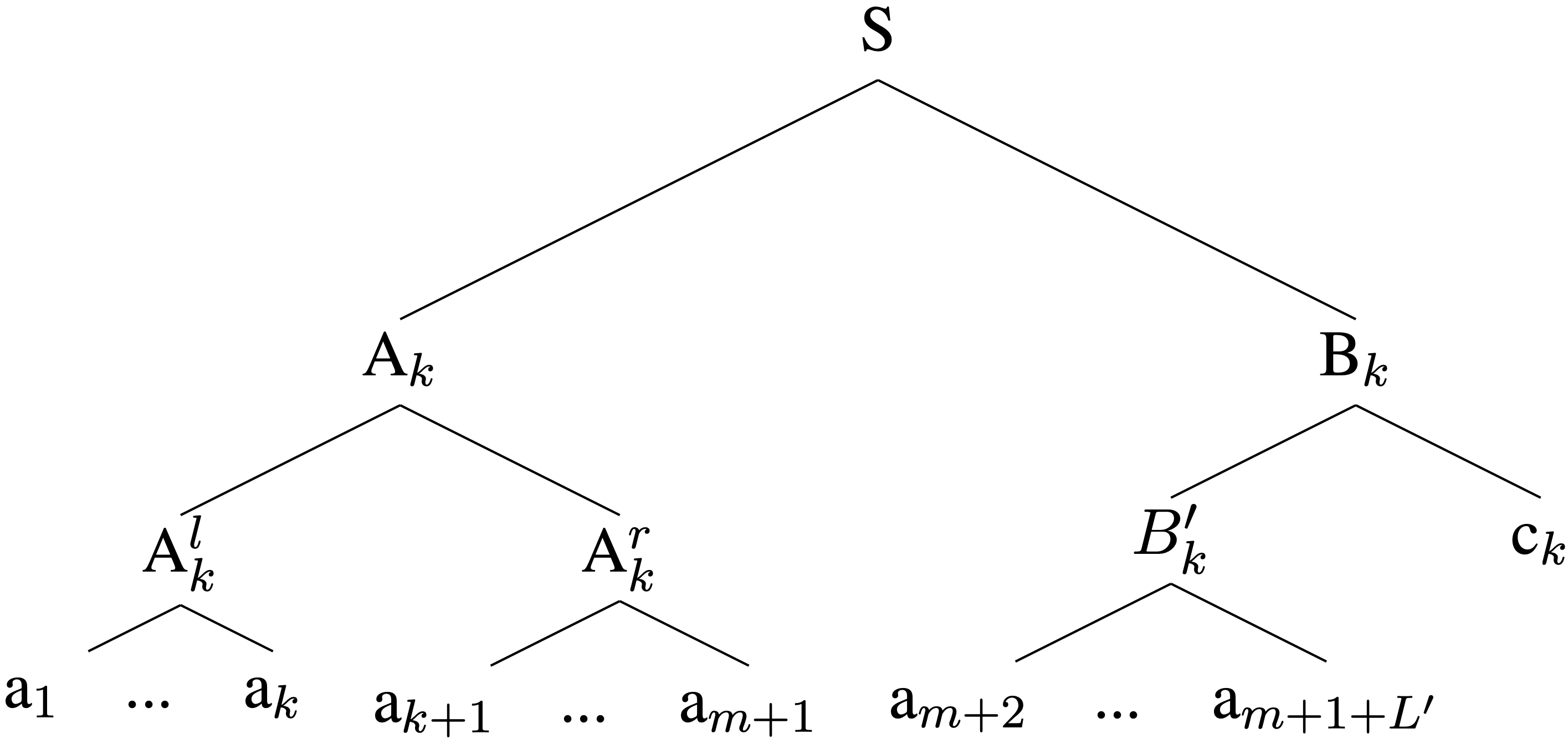}
    \caption{The structure of the parse tree of string $l_k = a_1 a_2 ... a_{m+1+L'} c_k \in L(G_{m,L'})$.
    Note that any $l_{k_1}$ and $l_{k_2}$ are almost the same except for the last token:
    the prefix $a_1 a_2 ... a_{m+1+L'}$ is shared among all strings in $L(G_{m,L'})$.
    However, their parses differ with respect to where $A_k$ is split. 
    The last token $c_k$ is unique to $l_k$ and hence determines the correct parse according to $G_{m,L'}$.
    }
    \label{fig:counterexample_parse}
\end{figure*}

The parse of an arbitrary $l_k$ is shown in Figure \ref{fig:counterexample_parse}.
Each $l_k$ corresponds to a unique parse determined by the choice of $k$.
The structure of this PCFG is such that for the parsing algorithms we consider that proceed in a ``left-to-right" fashion on $l_k $, 
before processing the last token $c_k$,
it cannot infer the syntactic structure of $a_1 a_2 ... a_{m+1}$ any better than randomly guessing one of the $m$ possibilities. This is the main intuition behind Theorems \ref{thm:left_inf_right_finite} and \ref{thm:seq_left_to_right_backtrack}.

\begin{remark}
    While our theorems focus on the limitation of ``left-to-right" parsing,
    a symmetric argument implies the same limitation of ``right-to-left" parsing.
    Thus, our claim is that \emph{unidirectional} context (in either direction) limits the expressive power of parsing models.
\end{remark}

% Related works
\section{Related Works} \label{sec:related_works}

Neural models for parsing were first successfully implemented for supervised settings, e.g.  \citep{vinyals2015grammar, dyer2016recurrent, shen2018straight}. Unsupervised tasks remained seemingly out of reach, until the proposal of the 
Parsing-Reading-Predict Network (PRPN) by \citet{shen2018neural}, 
whose performance was thoroughly verified by extensive experiments in \citep{htut2018grammar}. 
The follow-up paper \citep{shen2018ordered} introducing the ON-LSTM architecture simplified radically the architecture in \citep{shen2018neural}, while still ultimately attempting to fit a distance metric with the help of carefully designed master forget gates. 
Subsequent work by \citet{kim2019compound}
departed from the usual way neural techniques are integrated in NLP, with great success: they proposed a neural parameterization for the EM algorithm for learning a PCFG, but in a manner that leverages semantic information as well --- 
achieving a large improvement on unsupervised parsing tasks.\footnote{By virtue of not relying on bounded or unidirectional context, the Compound PCFG \citep{kim2019compound} eschews the techniques in our paper. Specifically, by employing a bidirectional LSTM inference network in the process of constructing a tree given a sentence, the parsing is no longer ``left-to-right".
}  

In addition to constituency parsing, dependency parsing is another common task for syntactic parsing,
but for our analyses on the ability of various approaches to represent the max-likelihood parse of sentences generated from PCFGs,
we focus on the task of constituency parsing.
Moreover, it's important to note that there is another line of work aiming to probe the ability of models trained without explicit syntactic consideration (e.g. BERT) to nevertheless discover some (rudimentary) syntactic elements   \citep{bisk2015probing, linzen2016assessing, choe2016parsing, kuncoro2018lstms, williams2018latent, goldberg2019assessing, htut2019attention, hewitt-manning-2019-structural, reif2019visualizing}. However, to-date, we haven't been able to extract parse trees achieving scores that are close to the oracle binarized trees on standard benchmarks \citep{kim2019unsupervised, kim2019compound}.  

Methodologically, our work is closely related to a long line of works aiming to characterize the representational power of neural models (e.g. RNNs, LSTMs) through the lens of formal languages and formal models of computation.
Some of the works of this flavor are empirical in nature (e.g. LSTMs have been shown to possess stronger abilities to recognize some context-free language and even some context-sensitive language, compared with simple RNNs \citep{gers2001lstm, suzgun2019lstm} or GRUs \citep{weiss2018practical, suzgun2019lstm});
some results are theoretical in nature (e.g. \citet{siegelmann1992rnn}'s proof  that with unbounded precision and unbounded time complexity, RNNs are Turing-complete; related results investigate RNNs with bounded precision and computation time \citep{weiss2018practical}, as well as memory  \citep{merrill2019sequential, hewitt2020rnns}. Our work contributes to this line of works, 
but focuses on the task of syntactic parsing instead.

% Preliminary backgrounds
\section{Preliminaries} \label{sec:prelim}

In this section, we define some basic concepts and introduce the architectures we will consider.

\subsection{Probabilistic context-free grammar} 

First recall several definitions around formal language, especially probabilistic context free grammar: 
\begin{definition}[Probabilistic context-free grammar (PCFG)]
    \label{def:pcfg}
    Formally, a PCFG \citep{chomsky1956three} is a 5-tuple $G = (\Sigma, N, S, R, \Pi)$ in which $\Sigma$ is the set of terminals, $N$ is the set of non-terminals, $S \in N$ is the start symbol, $R$ is the set of production rules of the form $r = (r_L \rightarrow r_R)$, where $r_L \in N$, $r_R$ is of the form $B_1 B_2 ... B_m$, $m \in \mathbb{Z_+}$, and $\forall i \in \{1,2,...,m\}, B_i \in (\Sigma \cup N)$. Finally, $\Pi: R \mapsto [0, 1]$ is the rule probability function, in which for any
    $$r = (A \rightarrow B_1 B_2 ... B_m) \in R,$$ 
    $\Pi(r)$ is the conditional probability 
    $$P(r_R = B_1 B_2 ... B_m \, | \, r_L = A).$$
\end{definition}

\begin{definition}[Parse tree]
    \label{def:parse_tree}
    Let $T_G$ denote the set of parse trees that $G$ can derive. 
    Each $t \in T_G$ is associated with $\verb|yield(t)| \in \Sigma^*$, the sequence of terminals composed of the leaves of $t$ and $P_T(t) \in [0, 1]$, the probability of the parse tree, defined by the product of the probabilities of the rules in the derivation of $t$.
\end{definition}

\begin{definition}[Language and sentence]
    \label{def:language_and_sentence}
    The language of $G$ is 
    $$L(G) = \{s \in \Sigma^*: \exists t \in T_G, \verb|yield(t)| = s\}.$$
    Each $s \in L(G)$ is called a sentence in $L(G)$, and is associated with the set of parses $T_G(s) = \{t \in T_G \, | \, \verb|yield(t)| = s\}$, the set of max likelihood parses, $\argmax_{t \in T_G(s)} P_T(t)$, 
    and its probability $P_S(s) = \sum_{t \in T_G(s)} P_T(t)$. 
\end{definition}

\begin{definition}[Chomsky normal form (CNF)]
    \label{def:cnf}
    A PCFG $G = (\Sigma, N, S, R, \Pi)$ is in CNF \citep{chomsky1959properties} if we require, in addition to Definition \ref{def:pcfg}, that each rule $r \in R$ is in the form $A \rightarrow B_1 B_2$ where $B_1, B_2 \in N \setminus \{S\}$; $A \rightarrow a$ where $a \in \Sigma, a \ne \epsilon$; or $S \rightarrow \epsilon$ which is only allowed if the empty string $\epsilon \in L(G)$. 
    
    Every PCFG $G$ can be converted into a PCFG $G'$ in CNF such that $L(G) = L(G')$ \citep{hopcroft2006book}.
\end{definition}

\subsection{Syntactic distance} \label{sec:prelim:distance}

The Parsing-Reading-Predict Networks (PRPN) \citep{shen2018neural} is one of the leading approaches to unsupervised constituency parsing.
The parsing network (which computes the parse tree, hence the only part we focus on in our paper) 
is a convolutional network that computes the syntactic distances $d_t = d(w_{t-1}, w_t)$ (defined in Section \ref{sec:overview:distance}) based on the past $L$ words.
A deterministic greedy tree induction algorithm 
is then used to produce a parse tree as follows.  
First, we split the sentence $w_1 ...w_n$ into two constituents, $w_1 ... w_{t-1}$ and $w_t ... w_n$, where $t \in \mbox{argmax} \{d_t\}_{t=2}^n$ and form the left and right subtrees of $t$. We recursively repeat this procedure for the newly created constituents. An algorithmic form of this procedure is included as Algorithm \ref{alg:tree_induction} in Appendix \ref{sec:appendix:tree_induction}.

Note that, due to the deterministic nature of the tree-induction process, the ability of PRPN to learn a PCFG is completely contingent upon learning a good syntactic distance.

\subsection{The ordered neuron architecture} \label{sec:prelim:on-lstm}

Building upon the idea of representing the syntactic information with a real-valued distance measure at each position, a simple extension is to associate each position with a learned vector, and then use the vector for syntactic parsing. The ordered-neuron LSTM (ON-LSTM, \citealp{shen2018ordered}) proposes that the nodes that are closer to the root in the parse tree generate a longer span of terminals, and therefore should be less frequently ``forgotten'' than nodes that are farther away from the root.
The difference in the frequency of forgetting is captured by a carefully designed master forget gate vector $\tilde{f}$, as shown in Figure \ref{fig:shen2019intuition} (in Appendix \ref{sec:appendix:shen2019intuition}). 
Formally: 

\begin{definition}[Master forget gates, \citealp{shen2018ordered}]
    \label{def:master_forget_gates}
    Given the input sentence $W = w_1 w_2 ... w_n$ and a trained ON-LSTM, 
    running the ON-LSTM on $W$ gives the master forget gates, which are a sequence of $D$-dimensional vectors $\{\tilde{f}_t\}_{t=1}^n$,
    in which at each position $t$,
    $\tilde{f}_t = \tilde{f}_t(w_1, ..., w_t) \in [0, 1]^D$. 
    Moreover, let $\tilde{f}_{t,j}$ represent the $j$-th dimension of $\tilde{f}_t$. 
    The ON-LSTM architectures requires that 
    $\tilde{f}_{t,1} = 0$, $\tilde{f}_{t,D} = 1$, and
    $$\forall i < j, \quad \tilde{f}_{t,i} \le \tilde{f}_{t,j}.$$
\end{definition}

When parsing a sentence, the real-valued master forget gate vector $\tilde{f}_t$ at each position $t$ is reduced to a single real number representing the syntactic distance $d_t$ at position $t$ (see \eqref{eqn:on-lstm_to_distance}) \citep{shen2018neural}. 
Then, use the syntactic distances to obtain a parse.

\subsection{Transition-based parsing} \label{sec:prelim:seq_model}

In addition to outputting a single real numbered distance or a vector at each position $t$,
a left-to-right model can also parse a sentence by outputting a sequence of ``transitions" at each position $t$, 
an idea proposed in some traditional parsing approaches \citep{sagae2005classifier, chelba1997structured, chelba2000structured},
and also some more recent neural parameterization \citep{dyer2016recurrent}.

We introduce several items of notation:
\begin{itemize}
    \item $z_i^t$: the $i$-th transition 
    performed when reading in $w_t$, the $t$-th token of the sentence $$W = w_1 w_2 ... w_n.$$
    \vspace{-8mm}
    \item $N_t$: the number of transitions 
    performed between reading in the token $w_t$ and reading in the next token $w_{t+1}$. 
    \vspace{-3mm}
    \item $Z_t$: the sequence of transitions 
    after reading in the prefix $w_1 w_2 ... w_t$ of the sentence.
    $$Z_t = \{(z_1^j, z_2^j, ..., z_{N_j}^j) \, | \, j=1..t \}.$$
    \vspace{-8mm}
    \item $Z$: the parse of the sentence $W$. 
    $Z = Z_n$.
\end{itemize}

We base our analysis on the approach introduced in the \emph{parsing} version of \citep{dyer2016recurrent},
though that work additionally proposes a \emph{generator} version. 
\footnote{\citet{dyer2016recurrent} additionally proposes some \emph{generator} transitions. For simplicity, we analyze the simplest form: we only allow the model to return one parse, composed of the \emph{parser} transitions, for a given input sentence.  
Note that this simplified variant still confers full representational power in the ``full context" setting (see Section \ref{sec:seq_model}). 
}

\begin{definition}[Transition-based parser] \label{def:left_to_right_tree_adjoining_parser}
    A \emph{transition-based parser} uses a stack (initialized to empty) and an input buffer (initialized with the sentence $w_1 ... w_t$).
    At each position $t$, 
    based on a context $c_t$,
    the parser outputs a sequence of parsing transitions $\{z_{i}^t\}_{i=1}^{N_t}$,
    where each $z_{i}^t$ can be one of the following transitions (Definition \ref{def:adjoin}).
    The parsing stops when the stack contains one single constituent, and the buffer is empty.
\end{definition}

\begin{definition}[Parser transitions, \citealp{dyer2016recurrent}]
    \label{def:adjoin}
    A parsing transition can be one of the following three types:
    \begin{itemize}
        \item NT(X) pushes a non-terminal X onto the stack.
        \item SHIFT: removes the first terminal from the input buffer and pushes onto the stack. 
        \item REDUCE: pops from the stack until an open non-terminal is encountered, then pops this non-terminal and assembles everything popped to form a new constituent, labels this new constituent using this non-terminal, and finally pushes this new constituent onto the stack.
    \end{itemize}
\end{definition}

In Appendix Section \ref{sec:appendix:parsing_transitions_examples}, we provide an example of parsing the sentence ``\textit{I drink coffee with milk}" using the set of transitions given by Definition \ref{def:adjoin}.

The different context specifications and the corresponding representational powers of the transition-based parser are discussed in Section \ref{sec:seq_model}.

% Results on Syntactic Distance
\section{Representational Power of Neural Syntactic Distance Methods} \label{sec:distance}

In this section we formalize the results on syntactic distance-based methods. 
Since 
the tree induction algorithm
always generates a binary tree, we consider only PCFGs in Chomsky normal form (CNF) (Definition \ref{def:cnf}) so that the max likelihood parse of a sentence is also a binary tree structure.

To formalize the notion of ``representing'' a PCFG, we introduce the following definition:

\begin{definition}[Representing PCFG with syntactic distance] \label{def:pcfg_rep}
    Let $G$ be any PCFG in Chomsky Normal Form. A syntactic distance function $d$ is said to be able to {\bf $p$-represent} $G$ if 
    for a set of sentences in $L(G)$ whose total probability is at least $p$, $d$ can correctly induce the tree structure of the max likelihood parse of these sentences 
    without ambiguity.
\end{definition}

\begin{remark}
    Ambiguities could occur when, for example, there exists $t$ such that $d_t = d_{t+1}$.
    In this case,
    the tree induction algorithm
    would have to break ties when determining the local structure for $w_{t-1} w_t w_{t+1}$.
    We preclude this possibility in Definition \ref{def:pcfg_rep}.
\end{remark}

In the least restrictive setting, the whole sentence $W$, as well as the position index $t$ can be taken into consideration when determining each $d_t$. 
We prove that under this setting, there is a syntactic distance measure that can represent any PCFG.

\begin{theorem}[Full context]
    \label{thm:tree_given_position_aware}
    Let $c_t = (W, t)$. For each PCFG $G$ in Chomsky normal form, there exists a syntactic distance measure 
    $d_t = d(w_{t-1}, w_t \, | \, c_t)$
    that can 1-represent $G$.
\end{theorem}

\begin{proof}
    For any sentence $s = s_1 s_2 ... s_n \in L(G)$,
    let $T$ be its max likelihood parse tree.
    Since $G$ is in Chomsky normal form, $T$ is a binary tree.
    We will describe an assignment of $\{d_t: 2 \le t \le n \}$ such that 
    their order matches the level at which the branches split in $T$. 
    Specifically, $\forall t \in [2, n]$,
    let $a_t$ denote the lowest common ancestor of $w_{t-1}$ and $w_t$ in $T$.
    Let $d'_t$ denote the shortest distance between $a_t$ and the root of $T$.
    Finally, let $d_t = n - d'_t$.
    As a result, 
    $\{d_t: 2 \le t \le n \}$ induces $T$.
\end{proof}

\begin{remark}
    Since any PCFG can be converted to Chomsky normal form \citep{hopcroft2006book}, 
    Theorem \ref{thm:tree_given_position_aware} implies that given the whole sentence and the position index as the context, 
    the syntactic distance has sufficient representational power to capture any PCFG.
    It does not state, however, that the whole sentence and the position are the minimal contextual information needed for representability nor does it address training (i.e. optimization) issues.
\end{remark}

On the flipside, we show that restricting the context even mildly can considerably decrease the representational power. Namely, we show that if context is bounded \emph{even in a single direction} (to the left or to the right), there are PCFGs on which any syntactic distance will perform poorly
\footnote{In Theorem \ref{thm:left_inf_right_finite} we prove the more typical case, i.e. unbounded left context and bounded right context. The other case, i.e. bounded left context and unbounded right context, can be proved symmetrically.}. 
(Note in the implementation \citep{shen2018neural} the context only considers a bounded window to the left.)

\begin{theorem}[Limitation of left-to-right parsing via syntactic distance]
    \label{thm:left_inf_right_finite}
    Let $w_0 = \langle S \rangle$ be the sentence start symbol.
    Let the context 
    $$c_t = (w_0, w_1, ..., w_{t+L'}).$$ 
    $\forall \epsilon > 0$,
    there exists a PCFG $G$ in Chomsky normal form, 
    such that any syntactic distance measure 
    $d_t = d(w_{t-1}, w_t \, | \, c_t)$
    cannot $\epsilon$-represent $G$.
\end{theorem}

\begin{proof}
    Let $m > 1/\epsilon$ be a positive integer. 
    Consider the PCFG $G_{m,L'}$ in Definition \ref{def:difficult_grammar}.
    
    For any $k \in [m]$, consider the string $l_k \in L(G_{m,L'})$.
    Note that in the parse tree of $l_k$, the rule $S \rightarrow A_k B_k$ is applied.
    Hence, $a_k$ and $a_{k+1}$ are the unique pair of adjacent non-terminals in $a_1 a_2 ... a_{m+1}$ whose lowest common ancestor is the closest to the root in the parse tree of $l_k$.
    Then, in order for the syntactic distance metric $d$ to induce the correct parse tree for $l_k$, 
    $d_k$ must be the unique maximum in $\{ d_t: 2 \le t \le {m+1} \}$.
    
    However, $d$ is restricted to be in the form 
    $$d_t = d(w_{t-1}, w_t \, | \, w_0, w_1, ..., w_{t+L'}).$$
    Note that $\forall 1 \le k_1 < k_2 \le {m}$, 
    the first $m+1+L'$ tokens of $l_{k_1}$ and $l_{k_2}$ are the same,
    which implies that the inferred syntactic distances
    $$\{ d_t: 2 \le t \le {m+1} \}$$
    are the same for $l_{k_1}$ and $l_{k_2}$ at each position $t$.
    Thus, 
    it is impossible for $d$ to induce the correct parse tree for both $l_{k_1}$ and $l_{k_2}$. 
    Hence, $d$ is correct on at most one $l_k \in L(G_{m,L'})$, 
    which corresponds to probability at most $1/m < \epsilon$.
    Therefore, $d$ cannot $\epsilon$-represent $G_{m,L'}$.
\end{proof}

\begin{remark}
    In the counterexample, there are only $m$ possible parse structures for the prefix $a_1 a_2 ... a_{m+1}$.
    Hence, the proved fact that the probability of being correct is at most $1/m$ means that under the restrictions of unbounded look-back and bounded look-ahead, the distance cannot do better than random guessing for this grammar.
\end{remark}

\begin{remark}
    The above Theorem \ref{thm:left_inf_right_finite} formalizes the intuition discussed in \citep{htut2018grammar} outlining an intrinsic limitation of only considering bounded context in one direction.  
    Indeed, for the PCFG constructed in the proof,  
    the failure is a function of the context, not because of the fact that we are using a distance-based parser. 
\end{remark}

Note that as a corollary of the above theorem, if there is no context ($c_t = \verb|null|$) or the context is both bounded and unidirectional, i.e.
$$c_t = w_{t-L} w_{t-L+1} ... w_{t-1} w_t,$$ 
then there is a PCFG that cannot be $\epsilon$-represented by any such $d$.

% Results on ON-LSTM
\section{Representational Power of the Ordered Neuron Architecture} \label{sec:on-lstm}

In this section, we formalize the results characterizing the representational power of the ON-LSTM architecture. 
The master forget gates of the ON-LSTM,
$\{\tilde{f}_t\}_{t=2}^n$ in which each $\tilde{f}_t \in [0, 1]^D$,
encode the hierarchical structure of a parse tree, 
and \citet{shen2018ordered} proposes to carry out unsupervised constituency parsing via a reduction from the gate vectors to syntactic distances by setting: 

\begin{equation}
    \label{eqn:on-lstm_to_distance}
    \hat{d}_t^f = D - \sum_{j=1}^{D} \tilde{f}_{t,j} \text{ for } t = 2..n
\end{equation}

First we show that the gates in ON-LSTM \emph{in principle} form a lossless representation of any parse tree.

\begin{theorem}[Lossless representation of a parse tree]
    \label{thm:on-lstm_lossless}
    For any sentence $W = w_1 w_2 ... w_n$ with parse tree $T$ in any PCFG in Chomsky normal form, 
    there exists a dimensionality $D \in \mathbb{Z}_+$,
    a sequence of vectors $\{\tilde{f_t}\}_{t=2}^n$ in which each $\tilde{f_t} \in [0, 1]^D$, 
    such that the estimated syntactic distances via \eqref{eqn:on-lstm_to_distance}
    induce the structure of $T$. 
\end{theorem}

\begin{proof}
    By Theorem \ref{thm:tree_given_position_aware}, there is a syntactic distance measure $\{d_t\}_{t=2}^n$ that induces the structure of $T$
    (such that $\forall t, d_t \neq d_{t+1}$). 
    
    For each $t = 2..n$, 
    set $\hat{d_t} = k$ if $d_t$ is the $k$-th smallest entry in $\{d_t\}_{t=2}^n$,
    breaking ties arbitrarily.
    Then, each $\hat{d_t} \in [1, n-1]$,
    and $\{\hat{d_t}\}_{t=2}^n$ also induces the structure of $T$. 
    
    Let $D = n-1$.
    For each $t = 2..n$, 
    let $\tilde{f_t} = (0, ..., 0, 1, ..., 1)$ whose lower $\hat{d_t}$ dimensions are 0 and higher $D - \hat{d_t}$ dimensions are 1.
    Then, 
    $$\hat{d}_t^f = D - \sum_{j=1}^{D} \tilde{f}_{t,j} = D - (D - \hat{d_t}) = \hat{d_t}.$$
    Therefore, the calculated $\{\hat{d}_t^f\}_{t=2}^n$ induces the structure of $T$. 
\end{proof}

Although Theorem \ref{thm:on-lstm_lossless} shows the ability of the master forget gates to perfectly represent any parse tree, 
a \emph{left-to-right} parsing can be proved to be unable to return the correct parse with high probability.
In the actual implementation in \citep{shen2018ordered}, the (real-valued) master forget gate vectors $\{\tilde{f_t}\}_{t=1}^n$
are produced by 
feeding the input sentence $W = w_1 w_2 ... w_n$ to a model
trained with a language modeling objective. 
In other words, $\tilde{f}_{t,j}$ is calculated as a function of $w_1, ..., w_t$, rather than the entire sentence.  

As such, this left-to-right parser is subject to similar limitations as in Theorem \ref{thm:left_inf_right_finite}: 

\begin{theorem}[Limitation of syntactic distance estimation based on ON-LSTM]
    \label{thm:on-lstm_to_distance}
    For any $\epsilon > 0$,
    there exists a PCFG $G$ in Chomsky normal form, such that 
    the syntactic distance measure calculated with \eqref{eqn:on-lstm_to_distance}, $\hat{d}_t^f$, cannot $\epsilon$-represent $G$.
\end{theorem}

\begin{proof}
    Since by Definition \ref{def:master_forget_gates},  $\tilde{f}_{t,j}$ is a function of $w_1, ..., w_t$,
    the estimated syntactic distance $\hat{d}_t^f$ is also a function of $w_1, ..., w_t$.
    By Theorem \ref{thm:left_inf_right_finite}, even with unbounded look-back context $w_1, ..., w_t$, there exists a PCFG for which the probability that $\hat{d}_t^f$ induces the correct parse is arbitrarily low. 
\end{proof}

% Results on transition-based
\section{Representational Power of Transition-Based Parsing} \label{sec:seq_model}

In this section, we analyze a transition-based parsing framework inspired by \citep{dyer2016recurrent, chelba2000structured, chelba1997structured}. 

Again, we proceed to say first that ``full context'' confers full representational power.   
Namely, using the terminology of Definition \ref{def:left_to_right_tree_adjoining_parser}, 
we let the context $c_t$ at each position $t$ be the whole sentence $W$ and the position index $t$.
Note that any parse tree can be generated by a sequence of transitions defined in Definition \ref{def:adjoin}. 
Indeed, \citet{dyer2016recurrent} describes an algorithm to find such a sequence of transitions via a ``depth-first, left-to-right traversal" of the tree.

Proceeding to limited context, in the setting of typical left-to-right parsing, 
the context $c_t$ consists of all current and past tokens $\{w_j\}_{j=1}^{t}$ 
and all previous parses $\{(z_{1}^j, ..., z_{N_j}^j)\}_{j=1}^{t}$.
We'll again prove even stronger negative results, where we allow an optional look-ahead to $L'$ input tokens to the right.

\begin{theorem}[Limitation of transition-based parsing without full context]
    \label{thm:seq_left_to_right_backtrack}
    For any $\epsilon > 0$,
    there exists a PCFG $G$ in Chomsky normal form, such that for any learned transition-based parser
    (Definition \ref{def:left_to_right_tree_adjoining_parser}) based on context
    $$c_t = (\{w_j\}_{j=1}^{t+L'}, \{(z_{1}^j, ..., z_{N_j}^j)\}_{j=1}^{t}),$$
    the sum of the probabilities of the sentences in $L(G)$ for which the 
    parser returns the maximum likelihood parse is less than $\epsilon$.
\end{theorem}

\begin{proof}
    Let $m > 1/\epsilon$ be a positive integer. 
    Consider the PCFG $G_{m,L'}$ in Definition \ref{def:difficult_grammar}.
    
    Note that $\forall k, S \rightarrow A_k B_k$ is applied to yield string $l_k$. 
    Then in the parse tree of $l_k$, $a_k$ and $a_{k+1}$ are the unique pair of adjacent terminals in $a_1 a_2 ... a_{m+1}$ whose lowest common ancestor is the closest to the root.
    Thus, different $l_k$ requires a different sequence of transitions within the first $m+1$ input tokens, i.e. $\{z_i^t\}_{i \ge 1, \, 1 \le t \le m+1}$.
    
    For each $w \in L(G_{m,L'})$, before the last token $w_{m+2+L'}$ is processed, based on the common prefix $w_1 w_2 ... w_{m+1+L'} = a_1 a_2 ... a_{m+1+L'}$, it is equally likely that $w = l_k, \forall k$, w. prob. $1/m$ each.
    
    Moreover, when processing $w_{m+1}$, the bounded look-ahead window of size $L'$ does not allow access to the final input token $a_{m+2+L'} = c_k$.
    
    Thus, $\forall 1 \le k_1 < k_2 \le {m}$, 
    it is impossible for the parser to return the correct parse tree for both $l_{k_1}$ and $l_{k_2}$ without ambiguity.
    Hence, the parse is correct on at most one $l_k \in L(G)$, 
    which corresponds to probability at most $1/m < \epsilon$.
\end{proof}

% Conclusion
\section{Conclusion} \label{sec:conclusion}

In this work, we considered the representational power of two frameworks for constituency parsing prominent in the literature, based on learning a syntactic distance and learning a sequence of iterative transitions to build the parse tree --- in the sandbox of PCFGs. 
In particular, we show that if the context for calculating distance/deciding on transitions is limited at least to one side (which is typically the case in practice for existing architectures), there are PCFGs for which no good distance metric/sequence of transitions can be chosen to construct the maximum likelihood parse. 

This limitation was already suspected in \citep{htut2018grammar} as a potential failure mode of leading neural approaches like \citep{shen2018neural,shen2018ordered} and we show formally that this is the case. The PCFGs with this property track the intuition that bounded context methods will have issues when the parse at a certain position depends heavily on latter parts of the sentence. 

The conclusions thus suggest re-focusing our attention on methods like \citep{kim2019compound} which have enjoyed greater success on tasks like unsupervised constituency parsing, and do not fall in the paradigm analyzed in our paper. 
A question of definite further interest is how to augment models that have been successfully scaled up (e.g. BERT) in a principled manner with syntactic information, such that they can capture syntactic structure (like PCFGs). The other question of immediate importance is to understand the interaction between the syntactic and semantic modules in neural architectures --- information is shared between such modules in various successful architectures, e.g. \citep{dyer2016recurrent, shen2018neural, shen2018ordered, kim2019compound}, and the relative pros and cons of doing this are not well understood. 
Finally, our paper purely focuses on representational power, and does not consider algorithmic and statistical aspects of training. 
As any model architecture is associated with its distinct optimization and generalization considerations, 
and natural language data necessitates the modeling of the interaction between syntax and semantics,
those aspects of considerations are well beyond the scope of our analysis in this paper using the controlled sandbox of PCFGs, and are interesting directions for future work.

\bibliographystyle{plainnat}
\bibliography{references, anthology.bib}

\begin{thebibliography}{33}
\providecommand{\natexlab}[1]{#1}
\providecommand{\url}[1]{\texttt{#1}}
\expandafter\ifx\csname urlstyle\endcsname\relax
  \providecommand{\doi}[1]{doi: #1}\else
  \providecommand{\doi}{doi: \begingroup \urlstyle{rm}\Url}\fi

\bibitem[Bisk and Hockenmaier(2015)]{bisk2015probing}
Yonatan Bisk and Julia Hockenmaier.
\newblock Probing the linguistic strengths and limitations of unsupervised
  grammar induction.
\newblock In \emph{Proceedings of the 53rd Annual Meeting of the Association
  for Computational Linguistics and the 7th International Joint Conference on
  Natural Language Processing (Volume 1: Long Papers)}, pages 1395--1404,
  Beijing, China, July 2015. Association for Computational Linguistics.
\newblock \doi{10.3115/v1/P15-1135}.
\newblock URL \url{https://www.aclweb.org/anthology/P15-1135}.

\bibitem[Chelba(1997)]{chelba1997structured}
Ciprian Chelba.
\newblock A structured language model.
\newblock In \emph{35th Annual Meeting of the Association for Computational
  Linguistics and 8th Conference of the {E}uropean Chapter of the Association
  for Computational Linguistics}, pages 498--500, Madrid, Spain, July 1997.
  Association for Computational Linguistics.
\newblock \doi{10.3115/976909.979681}.
\newblock URL \url{https://www.aclweb.org/anthology/P97-1064}.

\bibitem[Chelba and Jelinek(2000)]{chelba2000structured}
Ciprian Chelba and Frederick Jelinek.
\newblock Structured language modeling.
\newblock \emph{Computer Speech \& Language}, 14\penalty0 (4):\penalty0 283 --
  332, 2000.
\newblock ISSN 0885-2308.
\newblock \doi{https://doi.org/10.1006/csla.2000.0147}.
\newblock URL
  \url{http://www.sciencedirect.com/science/article/pii/S0885230800901475}.

\bibitem[Chen et~al.(2017)Chen, Zhu, Ling, Wei, Jiang, and
  Inkpen]{chen-etal-2017-enhanced}
Qian Chen, Xiaodan Zhu, Zhen-Hua Ling, Si~Wei, Hui Jiang, and Diana Inkpen.
\newblock Enhanced {LSTM} for natural language inference.
\newblock In \emph{Proceedings of the 55th Annual Meeting of the Association
  for Computational Linguistics (Volume 1: Long Papers)}, pages 1657--1668,
  Vancouver, Canada, July 2017. Association for Computational Linguistics.
\newblock \doi{10.18653/v1/P17-1152}.
\newblock URL \url{https://www.aclweb.org/anthology/P17-1152}.

\bibitem[Choe and Charniak(2016)]{choe2016parsing}
Do~Kook Choe and Eugene Charniak.
\newblock Parsing as language modeling.
\newblock In \emph{Proceedings of the 2016 Conference on Empirical Methods in
  Natural Language Processing}, pages 2331--2336, Austin, Texas, November 2016.
  Association for Computational Linguistics.
\newblock \doi{10.18653/v1/D16-1257}.
\newblock URL \url{https://www.aclweb.org/anthology/D16-1257}.

\bibitem[{Chomsky}(1956)]{chomsky1956three}
N.~{Chomsky}.
\newblock Three models for the description of language.
\newblock \emph{IRE Transactions on Information Theory}, 2\penalty0
  (3):\penalty0 113--124, 1956.
\newblock \doi{10.1109/TIT.1956.1056813}.

\bibitem[Chomsky(1959)]{chomsky1959properties}
Noam Chomsky.
\newblock On certain formal properties of grammars.
\newblock \emph{Information and Control}, 2\penalty0 (2):\penalty0 137 -- 167,
  1959.
\newblock ISSN 0019-9958.
\newblock \doi{https://doi.org/10.1016/S0019-9958(59)90362-6}.
\newblock URL
  \url{http://www.sciencedirect.com/science/article/pii/S0019995859903626}.

\bibitem[Devlin et~al.(2019)Devlin, Chang, Lee, and
  Toutanova]{devlin-etal-2019-bert}
Jacob Devlin, Ming-Wei Chang, Kenton Lee, and Kristina Toutanova.
\newblock {BERT}: Pre-training of deep bidirectional transformers for language
  understanding.
\newblock In \emph{Proceedings of the 2019 Conference of the North {A}merican
  Chapter of the Association for Computational Linguistics: Human Language
  Technologies, Volume 1 (Long and Short Papers)}, pages 4171--4186,
  Minneapolis, Minnesota, June 2019. Association for Computational Linguistics.
\newblock \doi{10.18653/v1/N19-1423}.
\newblock URL \url{https://www.aclweb.org/anthology/N19-1423}.

\bibitem[Dyer et~al.(2016)Dyer, Kuncoro, Ballesteros, and
  Smith]{dyer2016recurrent}
Chris Dyer, Adhiguna Kuncoro, Miguel Ballesteros, and Noah~A. Smith.
\newblock Recurrent neural network grammars.
\newblock In \emph{Proceedings of the 2016 Conference of the North {A}merican
  Chapter of the Association for Computational Linguistics: Human Language
  Technologies}, pages 199--209, San Diego, California, June 2016. Association
  for Computational Linguistics.
\newblock \doi{10.18653/v1/N16-1024}.
\newblock URL \url{https://www.aclweb.org/anthology/N16-1024}.

\bibitem[Gers and Schmidhuber(2001)]{gers2001lstm}
F.~Gers and J.~Schmidhuber.
\newblock Lstm recurrent networks learn simple context-free and
  context-sensitive languages.
\newblock \emph{IEEE transactions on neural networks}, 12 6:\penalty0 1333--40,
  2001.

\bibitem[Goldberg(2019)]{goldberg2019assessing}
Yoav Goldberg.
\newblock Assessing bert's syntactic abilities, 2019.

\bibitem[He et~al.(2020)He, Wang, and Zhang]{he-etal-2020-enhancing}
Qi~He, Han Wang, and Yue Zhang.
\newblock Enhancing generalization in natural language inference by syntax.
\newblock In \emph{Findings of the Association for Computational Linguistics:
  EMNLP 2020}, pages 4973--4978, Online, November 2020. Association for
  Computational Linguistics.
\newblock \doi{10.18653/v1/2020.findings-emnlp.447}.
\newblock URL \url{https://www.aclweb.org/anthology/2020.findings-emnlp.447}.

\bibitem[Hewitt and Manning(2019)]{hewitt-manning-2019-structural}
John Hewitt and Christopher~D. Manning.
\newblock {A} structural probe for finding syntax in word representations.
\newblock In \emph{Proceedings of the 2019 Conference of the North {A}merican
  Chapter of the Association for Computational Linguistics: Human Language
  Technologies, Volume 1 (Long and Short Papers)}, pages 4129--4138,
  Minneapolis, Minnesota, June 2019. Association for Computational Linguistics.
\newblock \doi{10.18653/v1/N19-1419}.
\newblock URL \url{https://www.aclweb.org/anthology/N19-1419}.

\bibitem[Hewitt et~al.(2020)Hewitt, Hahn, Ganguli, Liang, and
  Manning]{hewitt2020rnns}
John Hewitt, Michael Hahn, Surya Ganguli, Percy Liang, and Christopher~D.
  Manning.
\newblock {RNN}s can generate bounded hierarchical languages with optimal
  memory.
\newblock In \emph{Proceedings of the 2020 Conference on Empirical Methods in
  Natural Language Processing (EMNLP)}, pages 1978--2010, Online, November
  2020. Association for Computational Linguistics.
\newblock \doi{10.18653/v1/2020.emnlp-main.156}.
\newblock URL \url{https://www.aclweb.org/anthology/2020.emnlp-main.156}.

\bibitem[Hopcroft et~al.(2006)Hopcroft, Motwani, and Ullman]{hopcroft2006book}
John~E. Hopcroft, Rajeev Motwani, and Jeffrey~D. Ullman.
\newblock \emph{Introduction to Automata Theory, Languages, and Computation
  (3rd Edition)}.
\newblock Addison-Wesley Longman Publishing Co., Inc., USA, 2006.
\newblock ISBN 0321462254.

\bibitem[Htut et~al.(2018)Htut, Cho, and Bowman]{htut2018grammar}
Phu~Mon Htut, Kyunghyun Cho, and Samuel Bowman.
\newblock Grammar induction with neural language models: An unusual
  replication.
\newblock In \emph{Proceedings of the 2018 {EMNLP} Workshop {B}lackbox{NLP}:
  Analyzing and Interpreting Neural Networks for {NLP}}, pages 371--373,
  Brussels, Belgium, November 2018. Association for Computational Linguistics.
\newblock \doi{10.18653/v1/W18-5452}.
\newblock URL \url{https://www.aclweb.org/anthology/W18-5452}.

\bibitem[Htut et~al.(2019)Htut, Phang, Bordia, and Bowman]{htut2019attention}
Phu~Mon Htut, Jason Phang, Shikha Bordia, and Samuel~R. Bowman.
\newblock Do attention heads in bert track syntactic dependencies?
\newblock \emph{ArXiv}, abs/1911.12246, 2019.

\bibitem[Kim et~al.(2019{\natexlab{a}})Kim, Dyer, and Rush]{kim2019compound}
Yoon Kim, Chris Dyer, and Alexander Rush.
\newblock Compound probabilistic context-free grammars for grammar induction.
\newblock In \emph{Proceedings of the 57th Annual Meeting of the Association
  for Computational Linguistics}, pages 2369--2385, Florence, Italy, July
  2019{\natexlab{a}}. Association for Computational Linguistics.
\newblock \doi{10.18653/v1/P19-1228}.
\newblock URL \url{https://www.aclweb.org/anthology/P19-1228}.

\bibitem[Kim et~al.(2019{\natexlab{b}})Kim, Rush, Yu, Kuncoro, Dyer, and
  Melis]{kim2019unsupervised}
Yoon Kim, Alexander Rush, Lei Yu, Adhiguna Kuncoro, Chris Dyer, and G{\'a}bor
  Melis.
\newblock Unsupervised recurrent neural network grammars.
\newblock In \emph{Proceedings of the 2019 Conference of the North {A}merican
  Chapter of the Association for Computational Linguistics: Human Language
  Technologies, Volume 1 (Long and Short Papers)}, pages 1105--1117,
  Minneapolis, Minnesota, June 2019{\natexlab{b}}. Association for
  Computational Linguistics.
\newblock \doi{10.18653/v1/N19-1114}.
\newblock URL \url{https://www.aclweb.org/anthology/N19-1114}.

\bibitem[Kuncoro et~al.(2018)Kuncoro, Dyer, Hale, Yogatama, Clark, and
  Blunsom]{kuncoro2018lstms}
Adhiguna Kuncoro, Chris Dyer, John Hale, Dani Yogatama, Stephen Clark, and Phil
  Blunsom.
\newblock {LSTM}s can learn syntax-sensitive dependencies well, but modeling
  structure makes them better.
\newblock In \emph{Proceedings of the 56th Annual Meeting of the Association
  for Computational Linguistics (Volume 1: Long Papers)}, pages 1426--1436,
  Melbourne, Australia, July 2018. Association for Computational Linguistics.
\newblock \doi{10.18653/v1/P18-1132}.
\newblock URL \url{https://www.aclweb.org/anthology/P18-1132}.

\bibitem[Linzen et~al.(2016)Linzen, Dupoux, and Goldberg]{linzen2016assessing}
Tal Linzen, Emmanuel Dupoux, and Yoav Goldberg.
\newblock Assessing the ability of {LSTM}s to learn syntax-sensitive
  dependencies.
\newblock \emph{Transactions of the Association for Computational Linguistics},
  4:\penalty0 521--535, 2016.
\newblock \doi{10.1162/tacl_a_00115}.
\newblock URL \url{https://www.aclweb.org/anthology/Q16-1037}.

\bibitem[Merrill(2019)]{merrill2019sequential}
William Merrill.
\newblock Sequential neural networks as automata.
\newblock In \emph{Proceedings of the Workshop on Deep Learning and Formal
  Languages: Building Bridges}, pages 1--13, Florence, August 2019. Association
  for Computational Linguistics.
\newblock \doi{10.18653/v1/W19-3901}.
\newblock URL \url{https://www.aclweb.org/anthology/W19-3901}.

\bibitem[Pang et~al.(2019)Pang, Lin, and Smith]{pang2019improving}
Deric Pang, Lucy~H. Lin, and Noah~A. Smith.
\newblock Improving natural language inference with a pretrained parser, 2019.

\bibitem[Reif et~al.(2019)Reif, Yuan, Wattenberg, Viegas, Coenen, Pearce, and
  Kim]{reif2019visualizing}
Emily Reif, Ann Yuan, Martin Wattenberg, Fernanda~B Viegas, Andy Coenen, Adam
  Pearce, and Been Kim.
\newblock Visualizing and measuring the geometry of bert.
\newblock In H.~Wallach, H.~Larochelle, A.~Beygelzimer, F.~d\textquotesingle
  Alch\'{e}-Buc, E.~Fox, and R.~Garnett, editors, \emph{Advances in Neural
  Information Processing Systems}, volume~32, pages 8594--8603. Curran
  Associates, Inc., 2019.
\newblock URL
  \url{https://proceedings.neurips.cc/paper/2019/file/159c1ffe5b61b41b3c4d8f4c2150f6c4-Paper.pdf}.

\bibitem[Sagae and Lavie(2005)]{sagae2005classifier}
Kenji Sagae and Alon Lavie.
\newblock A classifier-based parser with linear run-time complexity.
\newblock In \emph{Proceedings of the Ninth International Workshop on Parsing
  Technology}, pages 125--132, Vancouver, British Columbia, October 2005.
  Association for Computational Linguistics.
\newblock URL \url{https://www.aclweb.org/anthology/W05-1513}.

\bibitem[Shen et~al.(2018{\natexlab{a}})Shen, Lin, Jacob, Sordoni, Courville,
  and Bengio]{shen2018straight}
Yikang Shen, Zhouhan Lin, Athul~Paul Jacob, Alessandro Sordoni, Aaron
  Courville, and Yoshua Bengio.
\newblock Straight to the tree: Constituency parsing with neural syntactic
  distance.
\newblock In \emph{Proceedings of the 56th Annual Meeting of the Association
  for Computational Linguistics (Volume 1: Long Papers)}, pages 1171--1180,
  Melbourne, Australia, July 2018{\natexlab{a}}. Association for Computational
  Linguistics.
\newblock \doi{10.18653/v1/P18-1108}.
\newblock URL \url{https://www.aclweb.org/anthology/P18-1108}.

\bibitem[Shen et~al.(2018{\natexlab{b}})Shen, Lin, wei Huang, and
  Courville]{shen2018neural}
Yikang Shen, Zhouhan Lin, Chin wei Huang, and Aaron Courville.
\newblock Neural language modeling by jointly learning syntax and lexicon.
\newblock In \emph{International Conference on Learning Representations},
  2018{\natexlab{b}}.
\newblock URL \url{https://openreview.net/forum?id=rkgOLb-0W}.

\bibitem[Shen et~al.(2019)Shen, Tan, Sordoni, and Courville]{shen2018ordered}
Yikang Shen, Shawn Tan, Alessandro Sordoni, and Aaron Courville.
\newblock Ordered neurons: Integrating tree structures into recurrent neural
  networks.
\newblock In \emph{International Conference on Learning Representations}, 2019.
\newblock URL \url{https://openreview.net/forum?id=B1l6qiR5F7}.

\bibitem[Siegelmann and Sontag(1992)]{siegelmann1992rnn}
Hava~T. Siegelmann and Eduardo~D. Sontag.
\newblock On the computational power of neural nets.
\newblock In \emph{Proceedings of the Fifth Annual Workshop on Computational
  Learning Theory}, COLT '92, page 440–449, New York, NY, USA, 1992.
  Association for Computing Machinery.
\newblock ISBN 089791497X.
\newblock \doi{10.1145/130385.130432}.
\newblock URL \url{https://doi.org/10.1145/130385.130432}.

\bibitem[Suzgun et~al.(2019)Suzgun, Belinkov, Shieber, and
  Gehrmann]{suzgun2019lstm}
Mirac Suzgun, Yonatan Belinkov, Stuart Shieber, and Sebastian Gehrmann.
\newblock {LSTM} networks can perform dynamic counting.
\newblock In \emph{Proceedings of the Workshop on Deep Learning and Formal
  Languages: Building Bridges}, pages 44--54, Florence, August 2019.
  Association for Computational Linguistics.
\newblock \doi{10.18653/v1/W19-3905}.
\newblock URL \url{https://www.aclweb.org/anthology/W19-3905}.

\bibitem[Vinyals et~al.(2015)Vinyals, Kaiser, Koo, Petrov, Sutskever, and
  Hinton]{vinyals2015grammar}
Oriol Vinyals, \L~ukasz Kaiser, Terry Koo, Slav Petrov, Ilya Sutskever, and
  Geoffrey Hinton.
\newblock Grammar as a foreign language.
\newblock In C.~Cortes, N.~Lawrence, D.~Lee, M.~Sugiyama, and R.~Garnett,
  editors, \emph{Advances in Neural Information Processing Systems}, volume~28,
  pages 2773--2781. Curran Associates, Inc., 2015.
\newblock URL
  \url{https://proceedings.neurips.cc/paper/2015/file/277281aada22045c03945dcb2ca6f2ec-Paper.pdf}.

\bibitem[Weiss et~al.(2018)Weiss, Goldberg, and Yahav]{weiss2018practical}
Gail Weiss, Yoav Goldberg, and Eran Yahav.
\newblock On the practical computational power of finite precision {RNN}s for
  language recognition.
\newblock In \emph{Proceedings of the 56th Annual Meeting of the Association
  for Computational Linguistics (Volume 2: Short Papers)}, pages 740--745,
  Melbourne, Australia, July 2018. Association for Computational Linguistics.
\newblock \doi{10.18653/v1/P18-2117}.
\newblock URL \url{https://www.aclweb.org/anthology/P18-2117}.

\bibitem[Williams et~al.(2018)Williams, Drozdov, and
  Bowman]{williams2018latent}
Adina Williams, Andrew Drozdov, and Samuel~R. Bowman.
\newblock Do latent tree learning models identify meaningful structure in
  sentences?
\newblock \emph{Transactions of the Association for Computational Linguistics},
  6:\penalty0 253--267, 2018.
\newblock \doi{10.1162/tacl_a_00019}.
\newblock URL \url{https://www.aclweb.org/anthology/Q18-1019}.

\end{thebibliography}

% Appendix
\newpage
\appendix

\section{Tree Induction Algorithm Based on Syntactic Distance}
\label{sec:appendix:tree_induction}

The following algorithm is proposed in \citep{shen2018neural} to create a parse tree based on a given syntactic distance.

\begin{algorithm}[h] 
    \SetAlgoLined
    \KwData{Sentence $W = w_1 w_2 ... w_n$, syntactic distances $d_t = d(w_{t-1}, w_t \, | \, c_t)$, $2 \le t \le n$}
    \KwResult{A parse tree for $W$}
    Initialize the parse tree with a single node $n_0 = w_1 w_2 ... w_n$\;
    \While{$\exists$ leaf node $n = w_i w_{i+1} ... w_j$ where $i < j$}{
        Find $k \in \argmax_{i+1 \le k \le j} d_k$ \;
        Create the left child $n_l$ and the right child $n_r$ of $n$ \;
        $n_l \leftarrow w_i w_{i+1} ... w_{k-1}$ \;
        $n_r \leftarrow w_k w_{k+1} ... w_j$ \;
    }  
    \Return The parse tree rooted at $n_0$.
    \caption{Tree induction based on syntactic distance}
    \label{alg:tree_induction}
\end{algorithm}

\section{ON-LSTM Intuition}
\label{sec:appendix:shen2019intuition}

See Figure \ref{fig:shen2019intuition} below,
which is excerpted from \citep{shen2018ordered}
with minor adaptation to the notation.

\begin{figure*}[ht]
    \centering
    \includegraphics[width=16cm]{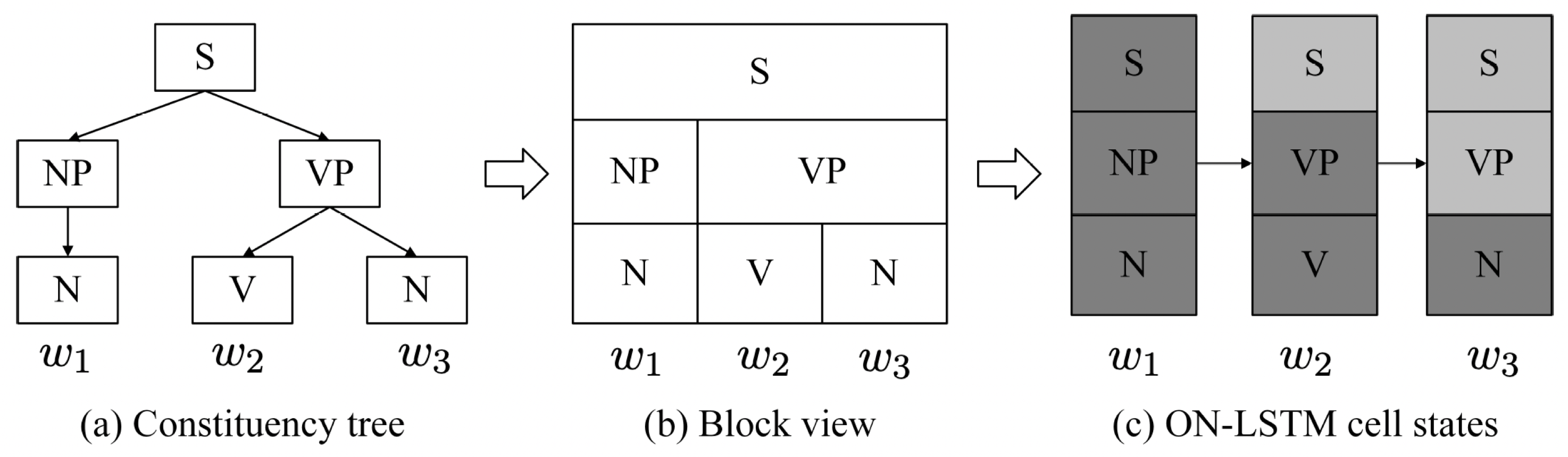}
    \caption{Relationship between the parse tree, the block view, and the ON-LSTM. Excerpted from \citep{shen2018ordered}
    with minor adaptation to the notation. 
    }
    \label{fig:shen2019intuition}
\end{figure*}

\newpage
\section{Examples of parsing transitions}
\label{sec:appendix:parsing_transitions_examples}

Table \ref{table:transition_based_parsing_example} below shows an example of parsing the sentence ``\textit{I drink coffee with milk}" using the set of transitions given by Definition \ref{def:adjoin}, 
which employs the parsing framework of \citep{dyer2016recurrent}.
The parse tree of the sentence is given by 

\vspace{-8mm}  % for arXiv only
{
\small
\begin{center}
    \begin{tabular}{c}
         \Tree [.S 
            [.NP  
                [.N I ] 
            ] 
            [.VP 
                [.V drink ] 
                [.NP
                    [.NP [.N coffee ] ]
                    [.PP 
                        [.P with ]
                        [.N milk ]
                    ]
                ] 
            ] 
        ]
    \end{tabular}
\end{center}
}

\vspace{-6mm}  % for arXiv only
\begin{table*}[ht]
\centering
\begin{tabular}{lll}
\hline
\textbf{Stack} & \textbf{Buffer} & \textbf{Action}\\
\hline
 & \textit{I drink coffee with milk} & \verb|NT(S)| \\
(S & \textit{I drink coffee with milk} & \verb|NT(NP)| \\
(S \(|\) (NP & \textit{I drink coffee with milk} & \verb|NT(N)| \\
(S \(|\) (NP \(|\) (N & \textit{I drink coffee with milk} & \verb|SHIFT| \\
(S \(|\) (NP \(|\) (N \(|\) \textit{I} & \textit{drink coffee with milk} & \verb|REDUCE| \\
(S \(|\) (NP (N \textit{I})) & \textit{drink coffee with milk} & \verb|NT(VP)| \\
(S \(|\) (NP (N \textit{I})) \(|\) (VP & \textit{drink coffee with milk} & \verb|NT(V)| \\
(S \(|\) (NP (N \textit{I})) \(|\) (VP \(|\) (V & \textit{drink coffee with milk} & \verb|SHIFT| \\
(S \(|\) (NP (N \textit{I})) \(|\) (VP \(|\) (V \(|\) \textit{drink} & \textit{coffee with milk} & \verb|REDUCE| \\
(S \(|\) (NP (N \textit{I})) \(|\) (VP \(|\) (V \textit{drink}) & \textit{coffee with milk} & \verb|NT(NP)| \\
(S \(|\) (NP (N \textit{I})) \(|\) (VP \(|\) (V \(|\) \textit{drink}) \(|\) \\ \quad (NP & \textit{coffee with milk} & \verb|NT(NP)| \\
(S \(|\) (NP (N \textit{I})) \(|\) (VP \(|\) (V \textit{drink}) \(|\) \\ \quad  (NP \(|\) (NP & \textit{coffee with milk} & \verb|NT(N)| \\
(S \(|\) (NP (N \textit{I})) \(|\) (VP \(|\) (V \textit{drink}) \(|\) \\ \quad  (NP \(|\) (NP \(|\) (N & \textit{coffee with milk} & \verb|SHIFT| \\
(S \(|\) (NP (N \textit{I})) \(|\) (VP \(|\) (V \textit{drink}) \(|\) \\ \quad  (NP \(|\) (NP \(|\) (N \(|\) \textit{coffee} & \textit{with milk} & \verb|REDUCE| \\
(S \(|\) (NP (N \textit{I})) \(|\) (VP \(|\) (V \textit{drink}) \(|\) \\ \quad  (NP \(|\) (NP (N \textit{coffee})) & \textit{with milk} & \verb|NT(PP)| \\
(S \(|\) (NP (N \textit{I})) \(|\) (VP \(|\) (V \textit{drink}) \(|\) \\ \quad  (NP \(|\) (NP (N \textit{coffee})) \(|\) (PP & \textit{with milk} & \verb|NT(P)| \\
(S \(|\) (NP (N \textit{I})) \(|\) (VP \(|\) (V \textit{drink}) \(|\) \\ \quad  (NP \(|\) (NP (N \textit{coffee})) \(|\) (PP \(|\) (P & \textit{with milk} & \verb|SHIFT| \\
(S \(|\) (NP (N \textit{I})) \(|\) (VP \(|\) (V \textit{drink}) \(|\) \\ \quad  (NP \(|\) (NP (N \textit{coffee})) \(|\) (PP \(|\) (P \(|\) \textit{with} & \textit{milk} & \verb|REDUCE| \\
(S \(|\) (NP (N \textit{I})) \(|\) (VP \(|\) (V \textit{drink}) \(|\) \\ \quad  (NP \(|\) (NP (N \textit{coffee})) \(|\) (PP \(|\) (P \textit{with}) & \textit{milk} & \verb|NT(N)| \\
(S \(|\) (NP (N \textit{I})) \(|\) (VP \(|\) (V \textit{drink}) \(|\) \\ \quad  (NP \(|\) (NP (N \textit{coffee})) \(|\) (PP \(|\) (P \textit{with}) \(|\) (N & \textit{milk} & \verb|SHIFT| \\
(S \(|\) (NP (N \textit{I})) \(|\) (VP \(|\) (V \textit{drink}) \(|\)  \\ \quad (NP \(|\) (NP (N \textit{coffee})) \(|\) (PP \(|\) (P \textit{with}) \(|\) (N \(|\) \textit{milk} &  & \verb|REDUCE| \\
(S (NP (N \textit{I})) (VP (V \textit{drink}) \\ \quad (NP (NP (N \textit{coffee})) (PP (P \textit{with}) (N \textit{milk}))))) &  &  \\
\hline
\end{tabular}
\caption{\label{table:transition_based_parsing_example}
Transition-based parsing of the sentence ``\textit{I drink coffee with milk}".
}
\end{table*}

\end{document}